\documentclass[runningheads,a4paper]{llncs}

\usepackage{graphicx}

\usepackage{url}

\urldef{\mailsa}\path|hshiino@yahoo-corp.jp|
\urldef{\mailsb}\path|sasaki@|
\urldef{\mailsc}\path|gang@|    
\urldef{\mailsd}\path|sugi@| 

\newcommand{\keywords}[1]{\par\addvspace\baselineskip
\noindent\keywordname\enspace\ignorespaces#1}

\usepackage[utf8]{inputenc} 
\usepackage[T1]{fontenc}    
\usepackage{hyperref}       
\usepackage{booktabs}       
\usepackage{amsfonts}       
\usepackage{nicefrac}       
\usepackage{microtype}      

\usepackage{subfigure} 
\renewcommand{\subfigtopskip}{0mm}
\renewcommand{\subfigbottomskip}{0mm}
\renewcommand{\subfigcapskip}{0mm}
\renewcommand{\subcapsize}{\footnotesize}

\usepackage{amsmath,amsthm,amssymb,amsfonts}

\usepackage{algorithm,algorithmic}

\renewcommand{\hat}{\widehat}
\renewcommand{\exp}[1]{\mathrm{exp}\left(#1\right)}

\newcommand{\argmin}{\mathop{\rm argmin}\limits}

\newcommand{\trans}[1]{#1^{\top}}
\newcommand{\inv}[1]{#1^{-1}}

\newcommand{\norm}[1]{\left\|#1\right\|}
\newcommand{\abs}[1]{\left|#1\right|}

\newcommand{\bc}{\boldsymbol{c}}

\newcommand{\be}{\boldsymbol{e}}

\newcommand{\bh}{\boldsymbol{h}}

\newcommand{\bn}{\boldsymbol{n}}

\newcommand{\bs}{\boldsymbol{s}}
\newcommand{\bt}{\boldsymbol{t}}
\newcommand{\bu}{\boldsymbol{u}}
\newcommand{\bv}{\boldsymbol{v}}
\newcommand{\bw}{\boldsymbol{w}}
\newcommand{\bx}{\boldsymbol{x}}
\newcommand{\by}{\boldsymbol{y}}

\newcommand{\balpha}{\boldsymbol{\alpha}}
\newcommand{\bbeta}{\boldsymbol{\beta}}

\newcommand{\btheta}{\boldsymbol{\theta}}

\newcommand{\bvarphi}{\boldsymbol{\varphi}}

\newcommand{\bpsi}{\boldsymbol{\psi}}

\newcommand{\bA}{\boldsymbol{A}}
\newcommand{\bB}{\boldsymbol{B}}

\newcommand{\bG}{\boldsymbol{G}}

\newcommand{\bI}{\boldsymbol{I}}

\newcommand{\bQ}{\boldsymbol{Q}}

\newcommand{\bS}{\boldsymbol{S}}

\newcommand{\bSigma}{\boldsymbol{\Sigma}}

\begin{document}

\mainmatter  

\title{Whitening-Free Least-Squares\\Non-Gaussian Component Analysis}

\titlerunning{Whitening-Free Least-Squares\\Non-Gaussian Component Analysis}

%
%
\author{Hiroaki Shiino$^1$
\and Hiroaki Sasaki$^2$ 
\and Gang Niu$^3$ 
\and Masashi Sugiyama$^{4,3}$ 
}
\authorrunning{Whitening-Free Least-Squares Non-Gaussian Component Analysis}

\institute{Yahoo Japan Corporation \\
Kioi Tower 1-3 Kioicho, Chiyoda-ku, Tokyo 102-8282, Japan.
\and Nara Institute of Science and Technology\\
8916-5 Takayama-cho Ikoma, Nara 630-0192, Japan.
\and The University of Tokyo\\
5-1-5 Kashiwanoha, Kashiwa-shi, Chiba 277-8561, Japan.
\and RIKEN Center for Advanced Intelligence Project\\
1-4-1 Nihonbashi, Chuo-ku, Tokyo 103-0027, Japan.
}
%
%

\toctitle{Lecture Notes in Computer Science}
\tocauthor{Authors' Instructions}
\maketitle

\begin{abstract}
\emph{Non-Gaussian component analysis} (NGCA) is an unsupervised linear dimension reduction method
that extracts low-dimensional non-Gaussian ``signals'' from 
high-dimensional data contaminated with Gaussian noise.
NGCA can be regarded as a generalization of \emph{projection pursuit} (PP)
and \emph{independent component analysis} (ICA)
to multi-dimensional and dependent non-Gaussian components.
Indeed, seminal approaches to NGCA are based on PP and ICA.
Recently, a novel NGCA approach called \emph{least-squares NGCA} (LSNGCA) has been developed, 
which gives a solution analytically
through least-squares estimation of \emph{log-density gradients} and eigendecomposition.
However, since \emph{pre-whitening} of data is involved in LSNGCA,
it performs unreliably when the data covariance matrix is ill-conditioned,
which is often the case in high-dimensional data analysis.
In this paper, we propose a \emph{whitening-free} variant of LSNGCA and
experimentally demonstrate its superiority.

\end{abstract}

\section{Introduction}
\label{introduction}

Dimension reduction is a common technique in high-dimensional data analysis
to mitigate the \emph{curse of dimensionality} \cite{vapnik1998statistical}.
Among various approaches to dimension reduction,
we focus on unsupervised linear dimension reduction in this paper.

It is known that the distribution of randomly projected data is close to Gaussian
\cite{Springer:Diederichs+:2013:SNGCA-SDP}.
Based on this observation,
\emph{non-Gaussian component analysis} (NGCA) \cite{JMLR:Blanchard+:2006:MIPP}
tries to find a subspace that contains 
non-Gaussian signal components
so that Gaussian noise components can be projected out.
NGCA is formulated in an elegant \emph{semi-parametric} framework
and non-Gaussian components can be extracted without specifying their distributions.
Mathematically, NGCA can be regarded as a generalization of
\emph{projection pursuit} (PP) \cite{IEEE:Friedman+:1974:Projection}
and \emph{independent component analysis} (ICA) \cite{NN:Hyvarinen+:2000:Independent}
to multi-dimensional and dependent non-Gaussian components.

The first NGCA algorithm is called \emph{multi-index PP} (MIPP).
PP algorithms such as FastICA~\cite{NN:Hyvarinen+:2000:Independent} use
a \emph{non-Gaussian index function} (NGIF) to find either a super-Gaussian or sub-Gaussian component.
MIPP uses a family of such NGIFs to find multiple non-Gaussian components
and apply principal component analysis (PCA) to extract a non-Gaussian subspace.
However, MIPP requires us to prepare appropriate NGIFs,
which is not necessarily straightforward in practice.
Furthermore, MIPP requires \emph{pre-whitening} of data, which can be
unreliable when the data covariance matrix is ill-conditioned.

To cope with these problems, MIPP has been extended in various ways.
The method called \emph{iterative metric adaptation for radial kernel functions} (IMAK)
\cite{Springer:Kawanabe+:2007:IMAK}
tries to avoid the manual design of NGIFs
by \emph{learning} the NGIFs from data
in the form of radial kernel functions.
However, this learning part is computationally highly expensive and pre-whitening is still necessary.
\emph{Sparse NGCA} (SNGCA) \cite{IEEE:Diederichs+:2010:SNGCA,Springer:Diederichs+:2013:SNGCA-SDP}
tries to avoid pre-whitening by imposing an appropriate constraint
so that the solution is independent of the data covariance matrix.
However, SNGCA involves \emph{semi-definite programming} which is computationally highly demanding,
and NGIFs still need to be manually designed.

Recently, a novel approach to NGCA called \emph{least-squares NGCA} (LSNGCA) has been proposed
\cite{AISTATS:Hiroaki+:2016:LSNGCA}.
Based on the \emph{gradient} of the log-density function,
LSNGCA constructs a vector that belongs to the non-Gaussian subspace
from each sample.
Then the method of \emph{least-squares log-density gradients} (LSLDG)
\cite{Springer:Cox:1985:Penalty,Springer:Sasaki+:2014:Clustering}
is employed to directly estimate the log-density gradient without density estimation.
Finally, the principal subspace of the set of vectors generated from all samples is extracted
by eigendecomposition.
LSNGCA is computationally efficient and no manual design of NGIFs is involved.
However, it still requires pre-whitening of data.

\begin{table}[t]
  \centering
 \caption{NGCA methods.}
   \label{tab:NGCA}
  \vspace*{-4mm}
  \small
  \begin{tabular}[t]{|@{\ }c@{\ }||@{\ }c@{\ }|@{\ }c@{\ }|@{\ }c@{\ }|@{\ }c@{\ }|@{\ }c@{\ }|c|}
\hline
&
MIPP&
IMAK&
SNGCA&
LSNGCA&
\begin{tabular}{@{}c@{}}
WF-LSNGCA\\
(proposed)\\
\end{tabular}\\\hline\hline
\begin{tabular}{@{}c@{}}
Manual \\
NGIF design\\
\end{tabular}&
Need&
\textbf{No Need}&
Need&
\textbf{No Need}&
\textbf{No Need}\\\hline
\begin{tabular}{@{}c@{}}
Computational \\
efficiency \\
\end{tabular}&
\begin{tabular}{@{}c@{}}
\textbf{Efficient}\\
\textbf{(iterative)}
\end{tabular}&
\begin{tabular}{@{}c@{}}
Inefficient\\
(iterative)
\end{tabular}&
\begin{tabular}{@{}c@{}}
Inefficient\\
(iterative)
\end{tabular}&
\begin{tabular}{@{}c@{}}
\textbf{Efficient}\\
\textbf{(analytic)}
\end{tabular}&
\begin{tabular}{@{}c@{}}
\textbf{Efficient}\\
\textbf{(analytic)}
\end{tabular}\\\hline
Pre-whitening&
Need&
Need&
\textbf{No Need}&
Need&
\textbf{No Need}\\\hline
  \end{tabular}
\vspace{-5mm}
\end{table}

The existing NGCA methods reviewed above are summarized in Table~\ref{tab:NGCA}.
In this paper, we propose a novel NGCA method that is
computationally efficient,
no manual design of NGIFs is involved,
and no pre-whitening is necessary.
Our proposed method is essentially an extention of LSNGCA
so that the covariance of data is implicitly handled without explicit pre-whitening
or explicit constraints.
Through experiments, we demonstrate that our proposed method,
called \emph{whitening-free} LSNGCA (WF-LSNGCA),
performs very well even when the data covariance matrix is ill-conditioned.

\section{Non-Gaussian Component Analysis} 
In this section, we formulate the problem of NGCA and
review the MIPP and LSNGCA methods.

\subsection{Problem Formulation}
\label{ngca_proplem}
Suppose that we are given a set of $d$-dimensional i.i.d.~samples of size $n$,
$\{\bx_i|\bx_i\in\mathbb{R}^d\}_{i=1}^{n}$,
which are generated by the following model:
\begin{align}
	\label{observation}
	\bx_i = \bA\bs_i+\bn_i,
\end{align}
where $\bs_i\in\mathbb{R}^m$ ($m\leq d$) is an $m$-dimensional signal vector
independently generated from an unknown non-Gaussian distribution
(we assume that $m$ is known),
$\bn_i\in\mathbb{R}^d$ is a noise vector
independently generated from a centered Gaussian distribution
with an unknown covariance matrix $\bQ$,
and $\bA\in\mathbb{R}^{d\times m}$ is an unknown mixing matrix of rank $m$.
Under this data generative model, 
probability density function $p(\bx)$ that samples $\{\bx_i\}_{i=1}^{n}$ follow
can be expressed in the following semi-parametric form \cite{JMLR:Blanchard+:2006:MIPP}:
\begin{align}
	\label{eq:semi-para}
	p(\bx) = f(\trans{\bB}\bx)\phi_{\bQ}(\bx),
\end{align}
where $f$ is an unknown smooth positive function on $\mathbb{R}^m$,
$\bB\in\mathbb{R}^{d\times m}$ is an unknown linear mapping,
$\phi_{\bQ}$ is the centered Gaussian density
with the covariance matrix $\bQ$, and $^{\top}$ denotes the transpose.
We note that decomposition \eqref{eq:semi-para} is not unique;
multiple combinations of $\bB$ and $f$ can give the same probability density function.
Nevertheless, the following $m$-dimensional subspace $E$,
called the \emph{non-Gaussian index space}, can be determined uniquely
\cite{Springer:Theis+:2006:Uniqueness}:
\begin{align}
	\label{eq:target-space}
	E = \mathrm{Null}(\trans{\bB})^{\perp}=\mathrm{Range}(\bB),
\end{align}
where $\mathrm{Null}(\trans{\bB})$ denotes the null space of $\trans{\bB}$,
$^{\perp}$ denotes the orthogonal complement,
and $\mathrm{Range}(\bB)$ denotes the column space of $\bB$.

The goal of NGCA is to estimate the non-Gaussian index space $E$ from 
samples $\{\bx_i\}_{i=1}^{n}$.

\subsection{Multi-Index Projection Pursuit (MIPP)}
\label{mipp}
MIPP \cite{JMLR:Blanchard+:2006:MIPP} is the first algorithm of NGCA.

Let us \emph{whiten} the samples $\{\bx_i\}_{i=1}^{n}$
so that their covariance matrix becomes identity:
\[
\by_i:=\bSigma^{-\frac{1}{2}}\bx_i,
\]
where $\bSigma$ is the covariance matrix of $\bx$.
In practice, $\bSigma$ is replaced by the sample covariance matrix.
Then, for an NGIF $h$, the following vector $\bbeta(h)$
was shown to belong to the non-Gaussian index space $E$ \cite{JMLR:Blanchard+:2006:MIPP}:
\[
  \bbeta(h):=\mathbb{E}\left[\by h(\by) - \nabla_{\by}h(\by)\right],
\]
where $\nabla_{\by}$ denotes the differential operator w.r.t.~$\by$
and $\mathbb{E}[\cdot]$ denotes the expectation over $p(\bx)$.
MIPP generates a set of such vectors
from various NGIFs $\left\{h_l\right\}^L_{l=1}$:
\begin{align}
  \label{eq:mipp-beta}
  \widehat{\bbeta}_l:=\frac{1}{n}\sum_{i=1}^{n}\left[\by_i h_l(\by_i) - \nabla_{\by}h_l(\by_i)\right],
\end{align}
where the expectation is estimated by the sample average.
Then $\widehat{\bbeta}_l$ is normalized as
\begin{align}
  \label{eq:mipp-beta-normalization}
  \widehat{\bbeta}_l \leftarrow {\widehat{\bbeta}_l}\Bigg/{\sqrt{\frac{1}{n}\sum^{n}_{i=1}\|\by_i h_l(\by_i)-\nabla_{\by}h_l(\by_i)\|^2-\|\widehat{\bbeta}_l\|^2}},
\end{align}
by which $\|\widehat{\bbeta}_l\|$ is proportional to its signal-to-noise ratio.
Then vectors $\widehat{\bbeta}_l$ with their norm less than a pre-specified threshold $\tau>0$ are eliminated.
Finally, PCA is applied to the remaining vectors $\widehat{\bbeta}_l$ to obtain an estimate of the non-Gaussian index space $E$.

The behavior of MIPP strongly depends on the choice of NGIF $h$.
To improve the performance, MIPP actively searches informative $h$ as follows.
First, the form of $h$ is restricted to
$h(\by) = s(\trans{\bw}\by)$, 
where $\bw\in\mathbb{R}^{d}$ denotes a unit-norm vector and $s$ is a smooth real function.
Then, estimated vector $\widehat{\bbeta}$ is written as
\[
	\widehat{\bbeta} = \frac{1}{n}\sum_{i=1}^{n}\left(\by_i s(\trans{\bw}\by_i)- s'(\trans{\bw}\by_i)\bw \right),
\]
where $s'$ is the derivative of $s$.
This equation is actually equivalent to a single iteration of the PP algorithm called \emph{FastICA}
\cite{NN:Hyvarinen+:1999:Fast}. Based on this fact, the parameter $\bw$ is optimized by iteratively applying the following update rule until convergence:
\[
\bw \leftarrow \frac{\sum_{i=1}^{n}\left(\by_i s(\trans{\bw}\by_i)- s'(\trans{\bw}\by_i)\bw\right)}
{\|\sum_{i=1}^{n}\left(\by_i s(\trans{\bw}\by_i)- s'(\trans{\bw}\by_i)\bw\right)\|}.
\]

The superiority of MIPP has been investigated both theoretically and experimentally
\cite{JMLR:Blanchard+:2006:MIPP}.
However, MIPP has the weaknesses that NGIFs should be manually designed and pre-whitening is necessary.

\subsection{Least-Squares Non-Gaussian Component Analysis (LSNGCA)}
\label{lsngca}
LSNGCA \cite{AISTATS:Hiroaki+:2016:LSNGCA}
is a recently proposed NGCA algorithm that does not require
manual design of NGIFs (Table~\ref{tab:NGCA}).
Here the algorithm of LSNGCA is reviewed, which will be used for
further developing a new method in the next section.

\paragraph{Derivation:}
For whitened samples $\{\by_i\}_{i=1}^{n}$,
the semi-parametric form of NGCA given in Eq.\eqref{eq:semi-para}
can be simplified as
\begin{align}
	\label{eq:whitened-semi-para}
	p(\by) = \widetilde{f}(\trans{\widetilde{\bB}}\by)\phi_{\bI_d}(\by),
\end{align}
where $\widetilde{f}$ is an unknown smooth positive function on $\mathbb{R}^m$ and
$\widetilde{\bB}\in\mathbb{R}^{d\times m}$ is an unknown linear mapping.
Under this simplified semi-parametric form,
the non-Gaussian index space $E$ can be represented as
\[
E=\bSigma^{-\frac{1}{2}}\mathrm{Range}(\widetilde{\bB}).
\] 

Taking the logarithm  and differentiating the both sides of Eq.\eqref{eq:whitened-semi-para}
w.r.t.~$\by$ yield
\begin{align}
	\label{eq:whitened-derivative-log}
	\nabla_{\by} \ln p(\by) + \by 
        &= \widetilde{\bB} \nabla_{\trans{\widetilde{\bB}}\by} \ln \widetilde{f}(\trans{\widetilde{\bB}}\by),
\end{align}
where $\nabla_{\trans{\widetilde{\bB}}\by}$ denotes the differential operator w.r.t.~$\trans{\widetilde{\bB}}\by$.
This implies that
\[
\bu(\by):=\nabla_{\by} \ln p(\by) + \by
\]
belongs to the non-Gaussian index space $E$.
Then applying eigendecomposition to $\sum_{i=1}^{n}\bu(\by_i)\bu(\by_i)^\top$
and extracting the $m$ leading eigenvectors allow us to recover $\mathrm{Range}(\widetilde{\bB})$.
In LSNGCA, the method of \emph{least-squares log-density gradients} (LSLDG) \cite{Springer:Cox:1985:Penalty,Springer:Sasaki+:2014:Clustering}
is used to estimate the log-density gradient $\nabla_{\by} \ln p(\by)$ included in $\bu(\by)$,
which is briefly reviewed below.

\paragraph{LSLDG:}
\label{subsec:LSLDG}
Let $\partial_j$ denote the differential operator w.r.t.~the $j$-th element of $\by$.
LSLDG fits a model $g^{(j)}(\by)$ to $\partial_j\ln p(\by)$, the $j$-th element of log-density gradient $\nabla_{\by}\ln p(\by)$,
under the squared loss:
\begin{align}
  J(g^{(j)}) &:= \mathbb{E}[(g^{(j)}(\by)-\partial_j\ln p(\by))^2]-\mathbb{E}[(\partial_j\ln p(\by))^2]\nonumber\\
&\phantom{:}= \mathbb{E}[g^{(j)}(\by)^2]-2\mathbb{E}[g^{(j)}(\by)\partial_j\ln p(\by)].
\label{eq:lsldg-approx-loss}
\end{align}
The second term in Eq.\eqref{eq:lsldg-approx-loss} yields
\begin{align*}
  \mathbb{E}[g^{(j)}(\by)\partial_j\ln p(\by)]
  &=\int g^{(j)}(\by)(\partial_j\ln p(\by))p(\by)\mathrm{d}\by
  =\int g^{(j)}(\by)\partial_jp(\by)\mathrm{d}\by\\
  &=-\int \partial_jg^{(j)}(\by)p(\by)\mathrm{d}\by
  =-\mathbb{E}[\partial_jg^{(j)}(\by)],
\end{align*}
where the second-last equation follows from \emph{integration by parts}
under the assumption $\lim_{|y^{(j)}|\rightarrow\infty}g^{(j)}(\by)p(\by)=0$.
Then sample approximation yields
\begin{align}
  J(g^{(j)}) &= \mathbb{E}[g^{(j)}(\by)^2-2\partial_jg^{(j)}(\by)]
\label{eq:lsldg-approx-loss2}
  \approx
  \frac{1}{n}\sum_{i=1}^{n}[g^{(j)}(\by_i)^2 + 2 \partial_j g^{(j)}(\by_i)].
\end{align}
As a model of the log-density gradient, LSLDG uses a linear-in-parameter form:
\begin{align}
  g^{(j)}(\by)=\sum_{k=1}^{b}\theta_{k,j}\psi_{k,j}(\by)=\trans{\btheta_{j}}\bpsi_{j}(\by),
\label{eq:model}
\end{align}
where $b$ denotes the number of basis functions, $\btheta_{j}:=\trans{(\theta_{1,j},\ldots,\theta_{b,j})}$ 
is a parameter vector to be estimated, and $\bpsi_{j}(\by):=\trans{(\psi_{1,j}(\by),\ldots,\psi_{b,j}(\by))}$ is a basis function vector.
The parameter vector $\btheta_j$ is learned by solving the following regularized empirical optimization problem:
\[  
\widehat{\btheta}_j = \argmin_{\btheta_j}\left[\trans{\btheta_j}\widehat{\bG}_j\btheta_j+2\trans{\btheta_j}\widehat{\bh}_j+\lambda_j\|\btheta_j\|^2\right],
\]
where $\lambda_j>0$ is the regularization parameter,
\begin{align*}
	\widehat{\bG}_j&=\frac{1}{n}\sum_{i=1}^{n}\bpsi_j(\by_i)\trans{\bpsi_j(\by_i)},~~
	\widehat{\bh}_j=\frac{1}{n}\sum_{i=1}^{n}\partial_j \bpsi_j(\by_i).
\end{align*}
This optimization problem can be analytically solved as 
\[
	\widehat{\btheta}_j = -\left(\widehat{\bG}_j+\lambda_j\bI_b\right)^{-1}\widehat{\bh}_j,
\]
where $\bI_b$ is the $b$-by-$b$ identity matrix.
Finally, an estimator of the log-density gradient $g^{(j)}(\by)$ is obtained as
\[
  \widehat{g}^{(j)}(\by) = \trans{\widehat{\btheta}_j}\bpsi_j(\by).
\]
All tuning parameters such as the regularization parameter $\lambda_j$ 
and parameters included in the basis function $\psi_{k,j}(\by)$
can be systematically chosen based on cross-validation w.r.t.~Eq.\eqref{eq:lsldg-approx-loss2}.

\section{Whitening-Free LSNGCA}

In this section, we propose a novel NGCA algorithm that does not involve pre-whitening.
A pseudo-code of the proposed method,
which we call \emph{whitening-free LSNGCA} (WF-LSNGCA),
is summarized in Algorithm~\ref{alg:wflsngca}.

\begin{algorithm}[t]
\caption{Pseudo-code of WF-LSNGCA.}
\label{alg:wflsngca}
\begin{algorithmic}[1]
\INPUT Element-wise standardized data samples: $\left\{\bx_i\right\}_{i=1}^{n}$.
\STATE Obtain an estimate $\widehat{\bv}(\bx)$
of $\bv(\bx) = \nabla_{\bx}\ln p(\bx) - \nabla_{\bx}^2\ln p(\bx)\bx$ by the method described in Section~\ref{subsec:bv(bx)}.
\STATE Apply eigendecomposition to $\sum_{i=1}^n\widehat{\bv}(\bx_i)\trans{\widehat{\bv}(\bx_i)}$
and extract the $m$ leading eigenvectors as an orthonormal basis of
non-Gaussian index space $E$.
\end{algorithmic}
\end{algorithm}

\subsection{Derivation}
Unlike LSNGCA which used the simplified semi-parametric form \eqref{eq:whitened-semi-para},
we directly use the original semi-parametric form \eqref{eq:semi-para}
without whitening.
Taking the logarithm and differentiating the both sides of Eq.\eqref{eq:semi-para} w.r.t.~$\bx$ yield
\begin{align}
	\label{eq:derivative-log}
	\nabla_{\bx}\ln p(\bx) + \inv{\bQ}\bx = \bB \nabla_{\trans{\bB}\bx} \ln f(\trans{\bB}\bx),
\end{align}
where $\nabla_{\bx}$ denotes the derivative w.r.t.~$\bx$ and $\nabla_{\trans{\bB}\bx}$ denotes the derivative w.r.t.~$\trans{\bB}\bx$.
Further taking the derivative of Eq.\eqref{eq:derivative-log} w.r.t.~$\bx$ yields
\begin{align}
	\label{eq:derivative2-log}
        \inv{\bQ} &= -\nabla_{\bx}^2\ln p(\bx) + \bB \nabla_{\trans{\bB}\bx}^2\ln f(\trans{\bB}\bx) \trans{\bB},
\end{align}
where $\nabla_{\bx}^2$ denotes the second derivative w.r.t.~$\bx$.
Substituting Eq.\eqref{eq:derivative2-log} back into Eq.\eqref{eq:derivative-log} yields
\begin{align}
&\nabla_{\bx} \ln p(\bx) -\nabla_{\bx}^2\ln p(\bx)\bx
= \bB \left(\nabla_{\trans{\bB}\bx} \ln f(\trans{\bB}\bx)-\nabla_{\trans{\bB}\bx}^2\ln f(\trans{\bB}\bx)\trans{\bB}\bx\right).
	\label{eq:main}
\end{align}
This implies that
\[
  \bv(\bx):=\nabla_{\bx} \ln p(\bx)-\nabla_{\bx}^2\ln p(\bx)\bx
\]
belongs to the non-Gaussian index space $E$.
Then we apply eigendecomposition to $\sum_{i=1}^n{\bv}(\bx_i)\trans{{\bv}(\bx_i)}$
and extract the $m$ leading eigenvectors as an orthonormal basis of
non-Gaussian index space $E$.

Now the remaining task is to approximate $\bv(\bx)$ from data,
which is discussed below.

\subsection{Estimation of $\bv(\bx)$}\label{subsec:bv(bx)}
Let $v^{(j)}(\bx)$ be the $j$-th element of $\bv(\bx)$:
\[
v^{(j)}(\bx)=\partial_j\ln p(\bx)-\trans{\left(\nabla_{\bx}\partial_j \ln p(\bx)\right)}\bx.
\]
To estimate $v^{(j)}(\bx)$,
let us fit a model $w^{(j)}(\bx)$ to it under the squared loss:
\begin{align}
  R(w^{(j)})
  &:= \mathbb{E}[(w^{(j)}(\bx)-v^{(j)}(\bx))^2]-\mathbb{E}[v^{(j)}(\bx)^2]\nonumber\\
&\phantom{:}= \mathbb{E}[w^{(j)}(\bx)^2]- 2\mathbb{E}[w^{(j)}(\bx)v^{(j)}(\bx)]\nonumber\\
&\phantom{:}= \mathbb{E}[w^{(j)}(\bx)^2]- 2\mathbb{E}[w^{(j)}(\bx)\partial_j\ln p(\bx)]
 + 2\mathbb{E}[w^{(j)}(\bx)\trans{\left(\nabla_{\bx}\partial_j \ln p(\bx)\right)}\bx].
\label{R}
\end{align}
The second term in Eq.\eqref{R} yields
\begin{align*}
  \mathbb{E}[w^{(j)}(\bx)\partial_j\ln p(\bx)]
  &=\int w^{(j)}(\bx)(\partial_j\ln p(\bx))p(\bx)\mathrm{d}\bx
  =\int w^{(j)}(\bx)\partial_j p(\bx)\mathrm{d}\bx\\
  &=-\int \partial_jw^{(j)}(\bx) p(\bx)\mathrm{d}\bx
  =-\mathbb{E}[\partial_jw^{(j)}(\bx)],
\end{align*}
where the second-last equation follows from \emph{integration by parts}
under the assumption $\lim_{|x^{(j)}|\rightarrow\infty}w^{(j)}(\bx)p(\bx)=0$.
$\partial_j \ln p(\bx)$ included in the third term in Eq.\eqref{R}
may be replaced with the LSLDG estimator $\widehat{g}^{(j)}(\bx)$
reviewed in Section~\ref{subsec:LSLDG}. Note that the LSLDG estimator is obtained with non-whitened data $\bx$ in this method. Then we have
\begin{align}
  R(w^{(j)})
  &\approx
\mathbb{E}[w^{(j)}(\bx)^2+2\partial_jw^{(j)}(\bx)
+ 2w^{(j)}(\bx)\trans{(\nabla_{\bx}\widehat{g}^{(j)}(\bx))}\bx]
\label{R-LSLDG}
\\
  &\approx\frac{1}{n}\sum_{i=1}^n[w^{(j)}(\bx_i)^2+2\partial_jw^{(j)}(\bx_i)
  +2w^{(j)}(\bx_i) \trans{(\nabla_{\bx}\widehat{g}^{(j)}(\bx_i))}\bx_i].
\nonumber
\end{align}
Here, let us employ the following linear-in-parameter model as $w^{(j)}(\bx)$:
\begin{align}
  w^{(j)}(\bx) := \sum_{k=1}^{t}\alpha_{k,j}\varphi_{k,j}(\bx)=\trans{\balpha_{j}}\bvarphi_{j}(\bx),
  \label{eq:proposed-model}
\end{align}
where $t$ denotes the number of basis functions, $\balpha_{j}:=\trans{(\alpha_{1,j},\ldots,\alpha_{t,j})}$ is a parameter vector to be estimated,
and $\bvarphi_{j}(\bx):=\trans{(\varphi_{1,j}(\bx),\ldots,\varphi_{t,j}(\bx))}$ is a basis function vector.
The parameter vector $\balpha_{j}$ is learned  by minimizing the following regularized empirical optimization problem:
\[
	\widehat{\balpha}_j = \argmin_{\balpha_j}\left[\trans{\balpha}_j\widehat{\bS}_j\balpha_j+2\trans{\balpha}_j\widehat{\bt}_j(\bx)+\gamma_j\|\balpha_j\|^2\right],
\]
where $\gamma_j>0$ is the regularization parameter,
\begin{align*}
	\widehat{\bS}_j&=\frac{1}{n}\sum_{i=1}^{n}\bvarphi_j(\bx_i)\trans{\bvarphi_j(\bx_i)},\\
	~~
        \widehat{\bt}_j&=\frac{1}{n}\sum_{i=1}^{n}\left(\partial_j \bvarphi_j(\bx_i)+ \bvarphi_j(\bx_i)
          \trans{\left(\nabla_{\bx}\widehat{g}^{(j)}(\bx_i)\right)}\bx_i\right).
\end{align*}
This optimization problem can be analytically solved as 
\[
	\widehat{\balpha}_j = -\left(\widehat{\bS}_j+\gamma_j\bI_b\right)^{-1}\widehat{\bt}_j.
\]
Finally, an estimator of $v^{(j)}(\bx)$ is obtained as
\[
	\widehat{v}^{(j)}(\bx) = \trans{\widehat{\balpha}}_j\bvarphi_j(\bx).
\]
All tuning parameters such as the regularization parameter $\gamma_j$
and parameters included in the basis function $\varphi_{k,j}(\by)$
can be systematically chosen based on cross-validation w.r.t.~Eq.\eqref{R-LSLDG}.


\subsection{Theoretical Analysis}

Here, we investigate the convergence rate of WF-LSNGCA in a parametric setting. 

Let $g^*(\bx)$ be the optimal estimate to $\nabla_{\bx}\ln p(\bx)$ given by LSLDG based on the linear-in-parameter model $g(\bx)$, and let
\begin{align*}
  \bS_j^*
  &=\mathbb{E}\left[\boldsymbol{\varphi}_j(\bx) \boldsymbol{\varphi}_j(\bx)^{\top}\right],
~~  \bt_j^*
  =\mathbb{E}\left[\partial_j \bvarphi_j(\bx)
  + \bvarphi_j(\bx) \trans{\left(\nabla_{\bx}g^{*(j)}(\bx)\right)}\bx \right],
\\
  \balpha_j^* &= \argmin\nolimits_{\balpha}
  \left\{\balpha^\top\bS_j^*\balpha +2\balpha^{\top}\bt_j^*
  +\gamma_j^*\balpha^\top\balpha\right\},
~~ w^{*(j)}(\bx) = \balpha_j^{*\top} \boldsymbol{\varphi}_j(\bx),
\end{align*}
where $(\bS_j^*+\gamma_j^*\bI_b)$ must be strictly positive definite. In fact, $\bS_j^*$ should already be strictly positive definite, and thus $\gamma_j^*=0$ is also allowed in our theoretical analysis.

We have the following theorem (its proof is given in Section~\ref{sec:proof-theorem-convergence}):
\begin{theorem}
  \label{thm:main-convergence}%
  As $n\to\infty$, for any $\bx$,
  \begin{equation*}
  \|\widehat{\bv}(\bx)-\bw^*(\bx)\|_2 =\mathcal{O}_p\left(n^{-1/2}\right),
  \end{equation*}
  provided that $\gamma_j$ for all $j$ converge in $\mathcal{O}(n^{-1/2})$ to $\gamma_j^*$, i.e.,
  $\lim_{n\to\infty}n^{1/2}|\gamma_j-\gamma_j^*|<\infty$.
\end{theorem}

Theorem~\ref{thm:main-convergence} is based on the theory of perturbed optimizations \cite{bonnans96,bonnans98} as well as the convergence of LSLDG shown in \cite{AISTATS:Hiroaki+:2016:LSNGCA}. It guarantees that for any $\bx$, the estimate $\widehat{\bv}(\bx)$ in WF-LSNGCA converges to the optimal estimate $\bw^*(\bx)$ based on the linear-in-parameter model $w(\bx)$, and it achieves the optimal parametric convergence rate $\mathcal{O}_p(n^{-1/2})$. Note that Theorem~\ref{thm:main-convergence} deals only with the estimation error, and the approximation error is not taken into account. Indeed, approximation errors exist in two places, from $\bw^*(\bx)$ to $\bv(\bx)$ in WF-LSNGCA itself and from $g^*(\bx)$ to $\nabla_{\bx}\ln p(\bx)$ in the plug-in LSLDG estimator. Since the original LSNGCA also relies on LSLDG, it cannot avoid the approximation error introduced by LSLDG. For this reason, the convergence of WF-LSNGCA is expected to be as good as LSNGCA.

Theorem~\ref{thm:main-convergence} is basically a theoretical guarantee that is similar to \emph{Part One in the proof of Theorem 1} in \cite{AISTATS:Hiroaki+:2016:LSNGCA}. Hence, based on Theorem~\ref{thm:main-convergence}, we can go along the line of \emph{Part Two in the proof of Theorem 1} in \cite{AISTATS:Hiroaki+:2016:LSNGCA} and obtain the following corollary.

\begin{corollary}
 For eigendecomposition, define matrices $\widehat{\mathbf{\Gamma}}=\frac{1}{n}\sum_{i=1}^n\widehat{\bv}(\bx_i)\trans{\widehat{\bv}(\bx_i)}$ and $\mathbf{\Gamma}^*=\mathbb{E}[\bw^*(\bx)\trans{\bw^*(\bx)}]$.
Given the estimated subspace $\widehat{E}$ based on $n$ samples and the optimal estimated subspace $E^*$ based on infinite data, denote by $\widehat{\mathbf{E}}\in\mathbb{R}^{d\times m}$ the matrix form of an arbitrary orthonormal basis of $\widehat{E}$ and by $\mathbf{E}^*\in\mathbb{R}^{d\times m}$ that of $E^*$. Define the distance between subspaces as
  \begin{equation*}
  \mathcal{D}(\widehat{E},E^*)
  =\inf\nolimits_{\widehat{\mathbf{E}},\mathbf{E}^*}
  \|\widehat{\mathbf{E}}-\mathbf{E}^*\|_\mathrm{Fro},
  \end{equation*}
  where $\|\cdot\|_\mathrm{Fro}$ stands for the Frobenius norm. Then, as $n\to\infty$,
  \begin{equation*}
  \mathcal{D}(\widehat{E},E^*)
  =\mathcal{O}_p\left(n^{-1/2}\right),
  \end{equation*}
  provided that
 $\gamma_j$ for all $j$ converge in $\mathcal{O}(n^{-1/2})$ to $\gamma_j^*$
and $\bvarphi_j(\bx)$ are well-chosen basis functions such that the first $m$ eigenvalues of $\mathbf{\Gamma}^*$ are neither $0$ nor $+\infty$.
\end{corollary}

\subsection{Proof of Theorem~\ref{thm:main-convergence}}
\label{sec:proof-theorem-convergence}

\paragraph{Step 1.}
First of all, we establish the growth condition (see \emph{Definition~6.1} in \cite{bonnans98}). Denote the expected and empirical objective functions by
\begin{align*}
  R_j^*(\balpha) &=
  \balpha^{\top}\bS_j^*\balpha +2\balpha^{\top}\bt_j^* +\gamma_j^*\balpha^{\top}\balpha,\\
  \widehat{R}_j(\balpha) &=
  \balpha^{\top}\widehat{\bS}_j\balpha +2\balpha^{\top}\widehat{\bt}_j +\gamma_j\balpha^{\top}\balpha.
\end{align*}
Then $\balpha_j^* = \argmin\nolimits_{\balpha}R_j^*(\balpha)$, $\widehat{\balpha}_j = \argmin\nolimits_{\balpha}\widehat{R}_j(\balpha)$, and we have

\begin{lemma}
  \label{lem:grow}%
  Let $\epsilon_j$ be the smallest eigenvalue of $(\bS_j^*+\gamma_j^*\bI_b)$, then the following second-order growth condition holds
  \begin{equation*}
    R_j^*(\balpha) \ge
    R_j^*(\balpha_j^*)+\epsilon_j\|\balpha-\balpha_j^*\|_2^2.
  \end{equation*}
\end{lemma}
\begin{proof}
  $R_j^*(\balpha)$ must be strongly convex with parameter at least $2\epsilon_j$. Hence,
  \begin{align*}
    R_j^*(\balpha)
    &\ge R_j^*(\balpha_j^*)
    +(\nabla R_j^*(\balpha_j^*))^\top(\balpha-\balpha_j^*)
    +\trans{(\balpha-\balpha_j^*)}(\bS_j^*+\gamma_j^*\bI_b)(\balpha-\balpha_j^*)\\
    &\ge R_j^*(\balpha_j^*)+\epsilon_j\|\balpha-\balpha_j^*\|_2^2,
  \end{align*}
  where we used the optimality condition $\nabla R_j^*(\balpha_j^*)=\boldsymbol{0}$.
\end{proof}

\paragraph{Step 2.}
Second, we study the stability (with respect to perturbation) of $R_j^*(\balpha)$ at $\balpha_j^*$. Let
\begin{equation*}
\boldsymbol{u}=\{\boldsymbol{u}_S\in\mathcal{S}_+^b,
\boldsymbol{u}_t\in\mathbb{R}^b,
u_\gamma\in\mathbb{R}\}
\end{equation*}
be a set of perturbation parameters, where $\mathcal{S}_+^b\subset\mathbb{R}^{b\times b}$ is the cone of $b$-by-$b$ symmetric positive semi-definite matrices. Define our perturbed objective function by
\begin{align*}
R_j(\balpha,\boldsymbol{u}) &=
\balpha^{\top}(\bS_j^*+\boldsymbol{u}_S)\balpha
+2\balpha^{\top}(\bt_j^*+\boldsymbol{u}_t)
+(\gamma_j^*+u_\gamma)\balpha^{\top}\balpha.
\end{align*}
It is clear that $R_j^*(\balpha)=R_j(\balpha,\boldsymbol{0})$, and then the stability of $R_j^*(\balpha)$ at $\balpha_j^*$ can be characterized as follows.

\begin{lemma}
  \label{lem:lips}%
  The difference function $R_j(\balpha,\boldsymbol{u})-R_j^*(\balpha)$ is Lipschitz continuous in $\balpha$ modulus
  \begin{equation*}
  \omega(\boldsymbol{u}) = \mathcal{O}
  (\|\boldsymbol{u}_S\|_\mathrm{Fro} +\|\boldsymbol{u}_t\|_2 +|u_\gamma|)
  \end{equation*}
  on a sufficiently small neighborhood of $\balpha_j^*$.
\end{lemma}
\begin{proof}
  The difference function is
  \[ R_j(\balpha,\boldsymbol{u})-R_j^*(\balpha)
  = \balpha^{\top}\boldsymbol{u}_S\balpha
  +2\balpha^{\top}\boldsymbol{u}_t
  +u_\gamma\balpha^{\top}\balpha, \]
  with a partial gradient
  \[ \frac{\partial}{\partial\balpha}(R_j(\balpha,\boldsymbol{u})-R_j^*(\balpha))
  = 2\boldsymbol{u}_S\balpha +2\boldsymbol{u}_t +2u_\gamma\balpha. \]
  Notice that due to the $\ell_2$-regularization in $R_j^*(\balpha)$, $\exists M>0$ such that $\|\balpha_j^*\|_2\le M$. Now given a $\delta$-ball of $\balpha_j^*$, i.e., $B_\delta(\balpha_j^*) = \{\balpha\mid\|\balpha-\balpha_j^*\|_2\le\delta\}$, it is easy to see that $\forall\balpha\in B_\delta(\balpha_j^*)$,
  \[ \|\balpha\|_2
  \le \|\balpha-\balpha_j^*\|_2 +\|\balpha_j^*\|_2
  \le \delta+M, \]
  and consequently
\begin{align*}
  &\left\|\frac{\partial}{\partial\balpha}
  (R_j(\balpha,\boldsymbol{u})-R_j^*(\balpha))\right\|_2
  \le 2(\delta+M)(\|\boldsymbol{u}_S\|_\mathrm{Fro}+|u_\gamma|)
  +2\|\boldsymbol{u}_t\|_2. 
\end{align*}
  This says that the gradient $\frac{\partial}{\partial\balpha}(R_j(\balpha,\boldsymbol{u})-R_j^*(\balpha))$ has a bounded norm of order $\mathcal{O}(\|\boldsymbol{u}_S\|_\mathrm{Fro}+\|\boldsymbol{u}_t\|_2+|u_\gamma|)$, and proves that the difference function $R_j(\balpha,\boldsymbol{u})-R_j^*(\balpha)$ is Lipschitz continuous on the ball $B_\delta(\balpha_j^*)$, with a Lipschitz constant of the same order.
\end{proof}

\paragraph{Step 3.}
Lemma~\ref{lem:grow} ensures the unperturbed objective $R_j^*(\balpha)$ grows quickly when $\balpha$ leaves $\balpha_j^*$; Lemma~\ref{lem:lips} ensures the perturbed objective $R_j(\balpha,\boldsymbol{u})$ changes slowly for $\balpha$ around $\balpha_j^*$, where the slowness is compared with the perturbation $\boldsymbol{u}$ it suffers. Based on Lemma~\ref{lem:grow}, Lemma~\ref{lem:lips}, and \emph{Proposition 6.1} in \cite{bonnans98},
\begin{equation*}
\|\widehat{\balpha}_j-\balpha_j^*\|_2
\le \frac{\omega(\boldsymbol{u})}{\epsilon_j}
= \mathcal{O}(\|\boldsymbol{u}_S\|_\mathrm{Fro}+\|\boldsymbol{u}_t\|_2+|u_\gamma|),
\end{equation*}
since $\widehat{\balpha}_j$ is the exact solution to $\widehat{R}_j(\balpha)=R_j(\balpha,\boldsymbol{u})$ given $\boldsymbol{u}_S=\widehat{\bS}_j-\bS_j^*$, $\boldsymbol{u}_t=\widehat{\bt}_j-\bt_j^*$, and $u_\gamma=\gamma_j-\gamma_j^*$.

According to the \emph{central limit theorem} (CLT), $\|\boldsymbol{u}_S\|_\mathrm{Fro}=\mathcal{O}_p(n^{-1/2})$. Consider $\widehat{\bt}_j-\bt_j^*$:
\begin{align*}
\widehat{\bt}_j-\bt_j^*
  &=\frac{1}{n}\sum_{i=1}^{n}\partial_j \bvarphi_j(\bx_i)
  -\mathbb{E}\left[ \partial_j \bvarphi_j(\bx) \right]
   +\frac{1}{n}\sum_{i=1}^{n}  \bvarphi_j(\bx_i)
  \trans{\left(\nabla_{\bx}\widehat{g}^{(j)}(\bx_i)\right)}\bx_i\\
  &\quad -\mathbb{E}\left[  \bvarphi_j(\bx)
  \trans{\left(\nabla_{\bx}g^{*(j)}(\bx)\right)}\bx \right].
\end{align*}
The first half is clearly $\mathcal{O}_p(n^{-1/2})$ due to CLT. For the second half, the estimate $\widehat{g}^{(j)}(\bx)$ given by LSLDG converges to $g^{*(j)}(\bx)$ for any $\bx$ in $\mathcal{O}_p(n^{-1/2})$ according to \emph{Part One in the proof of Theorem 1} in \cite{AISTATS:Hiroaki+:2016:LSNGCA}, and $\nabla_{\bx}\widehat{g}^{(j)}(\bx)$ converges to $\nabla_{\bx}g^{*(j)}(\bx)$ in the same order because the basis functions in $\bpsi_j(\bx)$ are all derivatives of Gaussian functions. Consequently,
\begin{align*}
&\frac{1}{n}\sum_{i=1}^{n}  \bvarphi_j(\bx_i)
\trans{\left(\nabla_{\bx}\widehat{g}^{(j)}(\bx_i)\right)}\bx_i
- \frac{1}{n}\sum_{i=1}^{n} \bvarphi_j(\bx_i)
\trans{\left(\nabla_{\bx}g^{*(j)}(\bx_i)\right)}\bx_i
=\mathcal{O}_p(n^{-1/2}), 
\end{align*}
since $\nabla_{\bx}\widehat{g}^{(j)}(\bx)$ converges to $\nabla_{\bx}g^{*(j)}(\bx)$ for any $\bx$ in $\mathcal{O}_p(n^{-1/2})$, and
\begin{align*}
&\frac{1}{n}\sum_{i=1}^{n} \bvarphi_j(\bx_i)
\trans{\left(\nabla_{\bx}g^{*(j)}(\bx_i)\right)}\bx_i
-\mathbb{E}\left[ \bvarphi_j(\bx)
\trans{\left(\nabla_{\bx}g^{*(j)}(\bx)\right)}\bx \right]
=\mathcal{O}_p(n^{-1/2})
\end{align*}
due to CLT, which proves $\|\boldsymbol{u}_t\|_2=\mathcal{O}_p(n^{-1/2})$. Furthermore, we have already assumed that $|u_\gamma|=\mathcal{O}(n^{-1/2})$. Hence, as $n\to\infty$,
\begin{equation*}
\|\widehat{\balpha}_j-\balpha_j^*\|_2
= \mathcal{O}_p\left(n^{-1/2}\right).
\end{equation*}

\paragraph{Step 4.}
Finally, for any $\bx$, the gap of $\widehat{v}^{(j)}(\bx)$ and $w^{*(j)}(\bx)$ is bounded by
\begin{align*}
|\widehat{v}^{(j)}(\bx)-w^{*(j)}(\bx)|
\le \|\widehat{\balpha}_j-\balpha_j^*\|_2
\cdot \|\bvarphi_j(\bx)\|_2,
\end{align*}
where the \emph{Cauchy-Schwarz inequality} is used. Since the basis functions in $\bvarphi_j(\bx)$ are again all derivatives of Gaussian functions, $\|\bvarphi_j(\bx)\|_2$ must be bounded uniformly, and then
\begin{equation*}
|\widehat{v}^{(j)}(\bx)-w^{*(j)}(\bx)|
\le \mathcal{O}(\|\widehat{\balpha}_j-\balpha_j^*\|_2)
= \mathcal{O}_p\left(n^{-1/2}\right).
\end{equation*}
Applying the same argument for all $j=1,\ldots,d$ completes the proof. \qed

\section{Experiments}

In this section, we experimentally investigate the performance of MIPP, LSNGCA, and WF-LSNGCA.\footnote{
The source code of the experiments is at~\url{https://github.com/hgeno/WFLSNGCA}.
}

\subsection{Configurations of NGCA Algorithms}
\label{appendix:configurations}

\subsubsection{MIPP}
We use the MATLAB script which was used in the original MIPP paper \cite{JMLR:Blanchard+:2006:MIPP}\footnote{\url{http://www.ms.k.u-tokyo.ac.jp/software.html}}.
In this script, NGIFs of the form $s^k_m(z)$ ($m=1,\ldots,1000, k=1,\ldots,4$) are used:
\begin{align*}
s^1_{m}(z) &= z^3\exp{-\frac{z^2}{2\sigma_m^2}},~~~
s^2_{m}(z) = \mathrm{tanh}\left(a_m z\right), ~~~\\
s^3_{m}(z) &= \sin(b_m z),~~~
s^4_{m}(z) = \cos(b_m z),
\end{align*}
where $\sigma_m$, $a_m$, and $b_m$ are scalars chosen at the regular intervals from $\sigma_m\in[0.5,5]$, $a_m\in[0.05,5]$, and $b_m\in[0.05,4]$.
The cut-off threshold $\tau$ is set at $1.6$ and the number of FastICA iterations is set at $10$
(see Section~\ref{mipp}).

\subsubsection{LSNGCA}
Following \cite{Springer:Sasaki+:2014:Clustering},
the derivative of the Gaussian kernel is used as the basis function $\psi_{k,j}(\by)$ in the linear-in-parameter model \eqref{eq:model}:
\begin{align*}
\psi_{k,j}(\by) = \partial_j\exp{-\frac{\norm{\by-\bc_k}^2}{2\sigma_j^2}},
\end{align*}
where $\sigma_j>0$ is the Gaussian bandwidth and $\bc_k$ is the Gaussian center randomly selected from the whitened data samples $\{\by_i\}_{i=1}^{n}$. 
The number of basis functions is set at $b=100$.
For model selection, $5$-fold cross-validation is performed with respect to the hold-out 
error of Eq.\eqref{eq:lsldg-approx-loss2}
using $10$ candidate values at the regular intervals in logarithmic scale
for Gaussian bandwidth $\sigma_j\in[10^{-1},10^{1}]$ and regularization parameter $\lambda_j\in[10^{-5},10^1]$.

\subsubsection{WF-LSNGCA}
Similarly to LSNGCA, the derivative of the Gaussian kernel is used as the basis function 
$\varphi_{k,j}(\bx)$ in the linear-in-parameter model \eqref{eq:proposed-model} and the number of basis functions is set as $t=b=100$.
For model selection, $5$-fold cross-validation is performed with respect to the hold-out error of Eq.\eqref{R-LSLDG}
in the same way as LSNGCA.

\subsection{Artificial Datasets}
\label{sec:artificial}


\begin{figure}[t]
\centering
\subfigure[Independent Gaussian Mixture]{
\includegraphics[width = 0.225\textwidth]{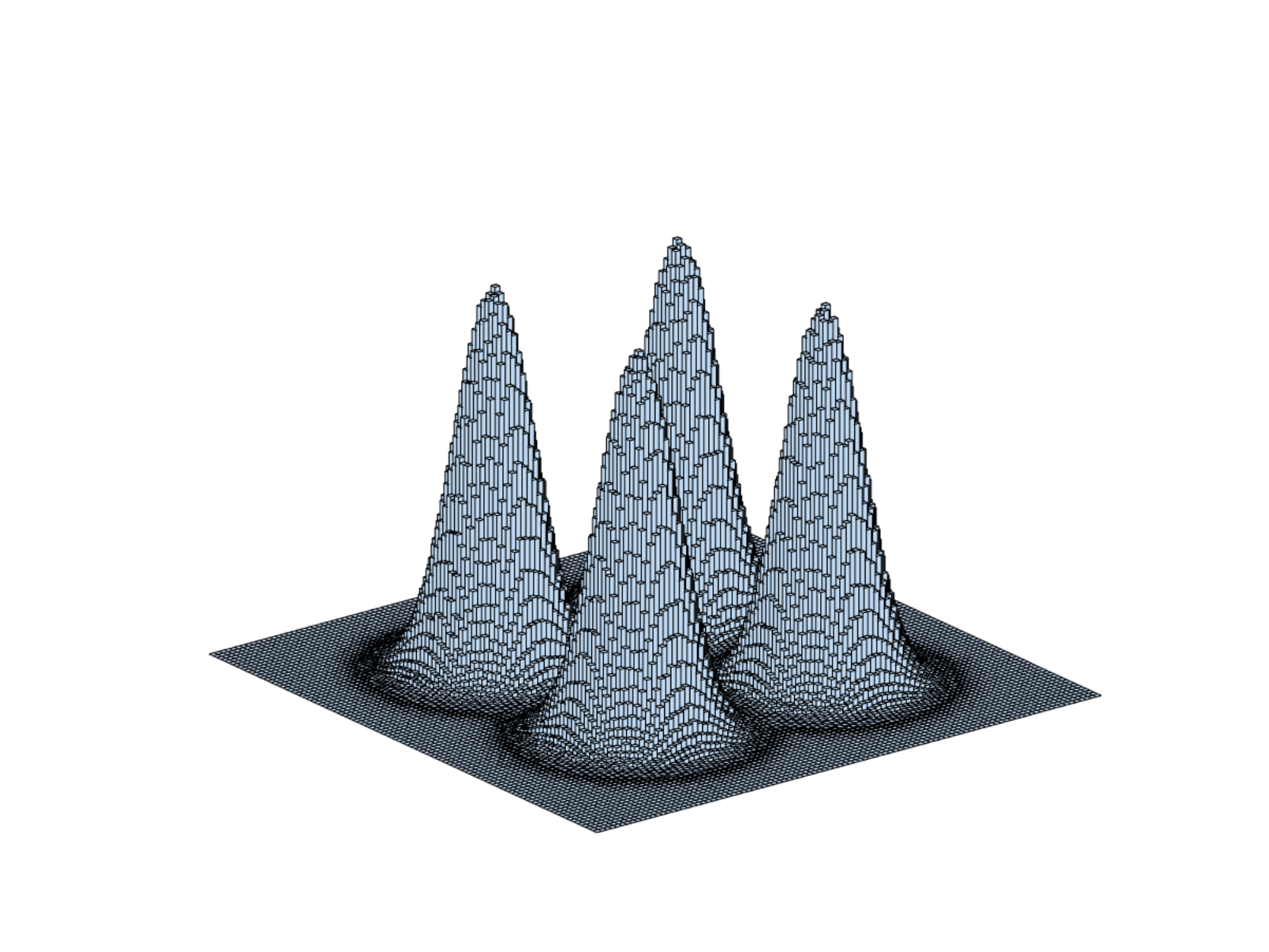}
}
\subfigure[Dependent Super-Gaussian]{
\includegraphics[width = 0.225\textwidth]{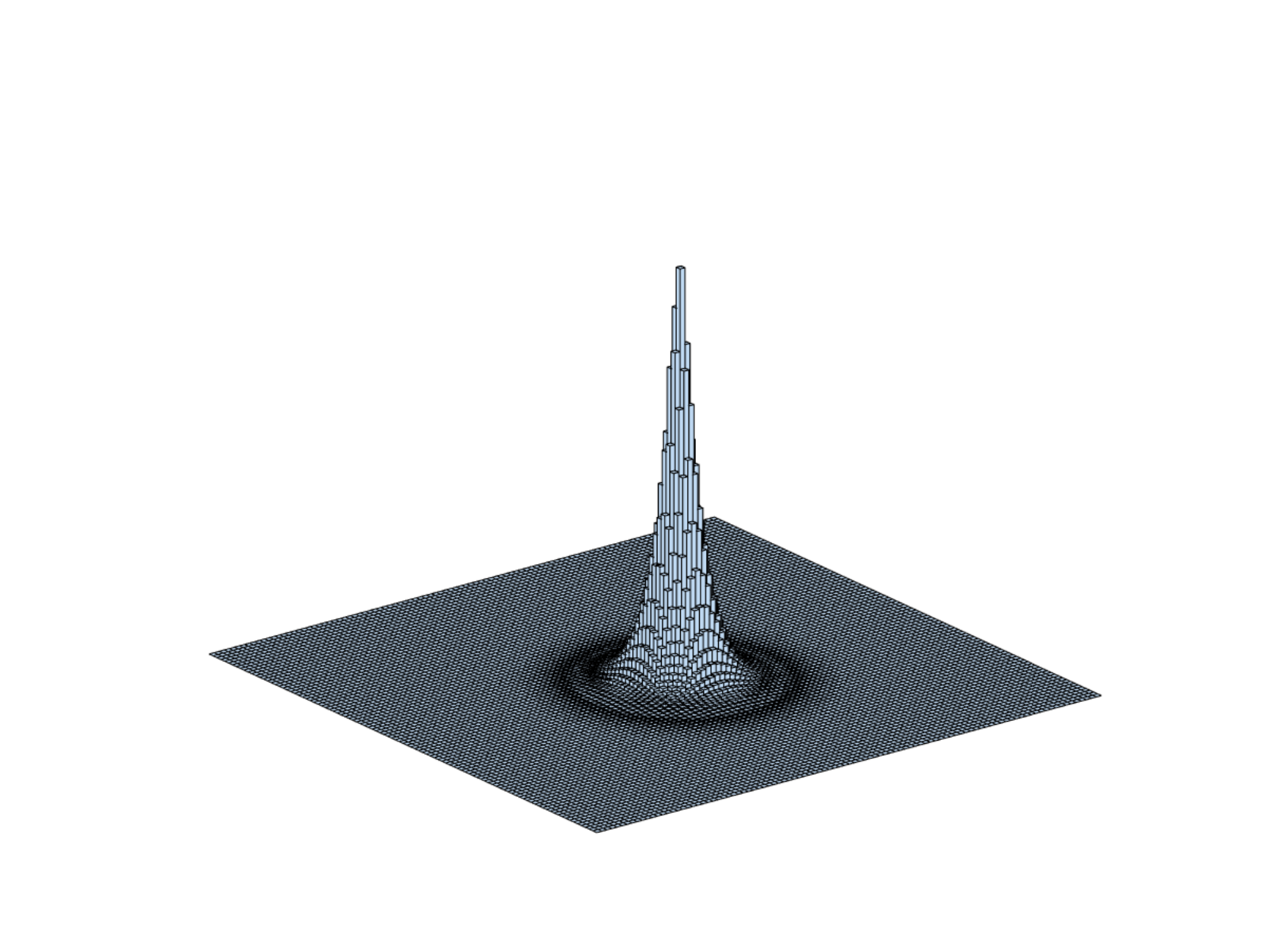}
}
\subfigure[Dependent Sub-Gaussian]{
\includegraphics[width = 0.225\textwidth]{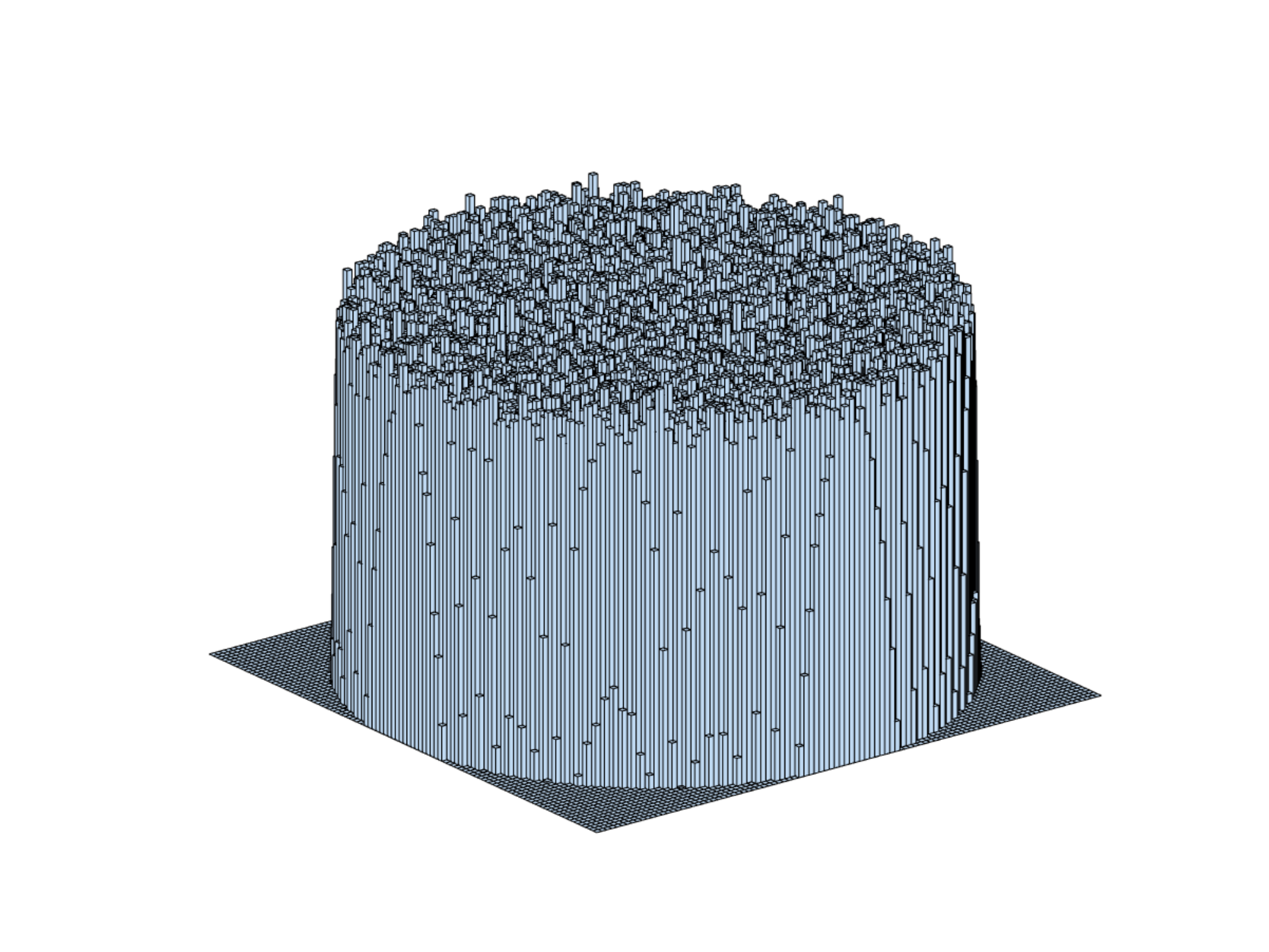}
}
\subfigure[Dependent Super- and Sub-Gaussian]{
\includegraphics[width = 0.225\textwidth]{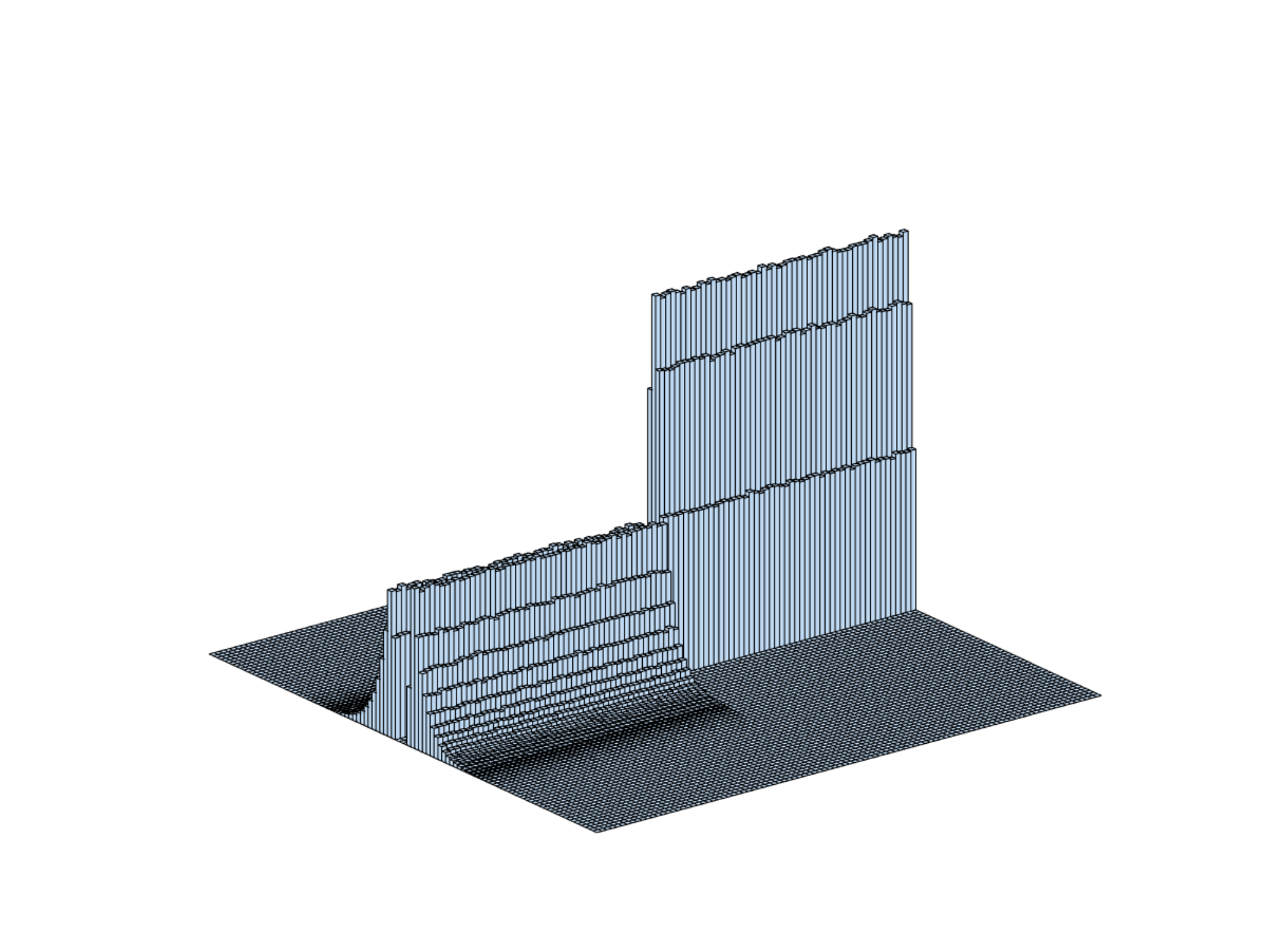}
}
\caption{Distributions of non-Gaussian components.}
\label{fig:Dist}
\vspace{-2mm}
\end{figure}
\begin{figure}[t]
\centering
\subfigure[Independent Gaussian Mixture]{
\includegraphics[width = 0.225\textwidth]{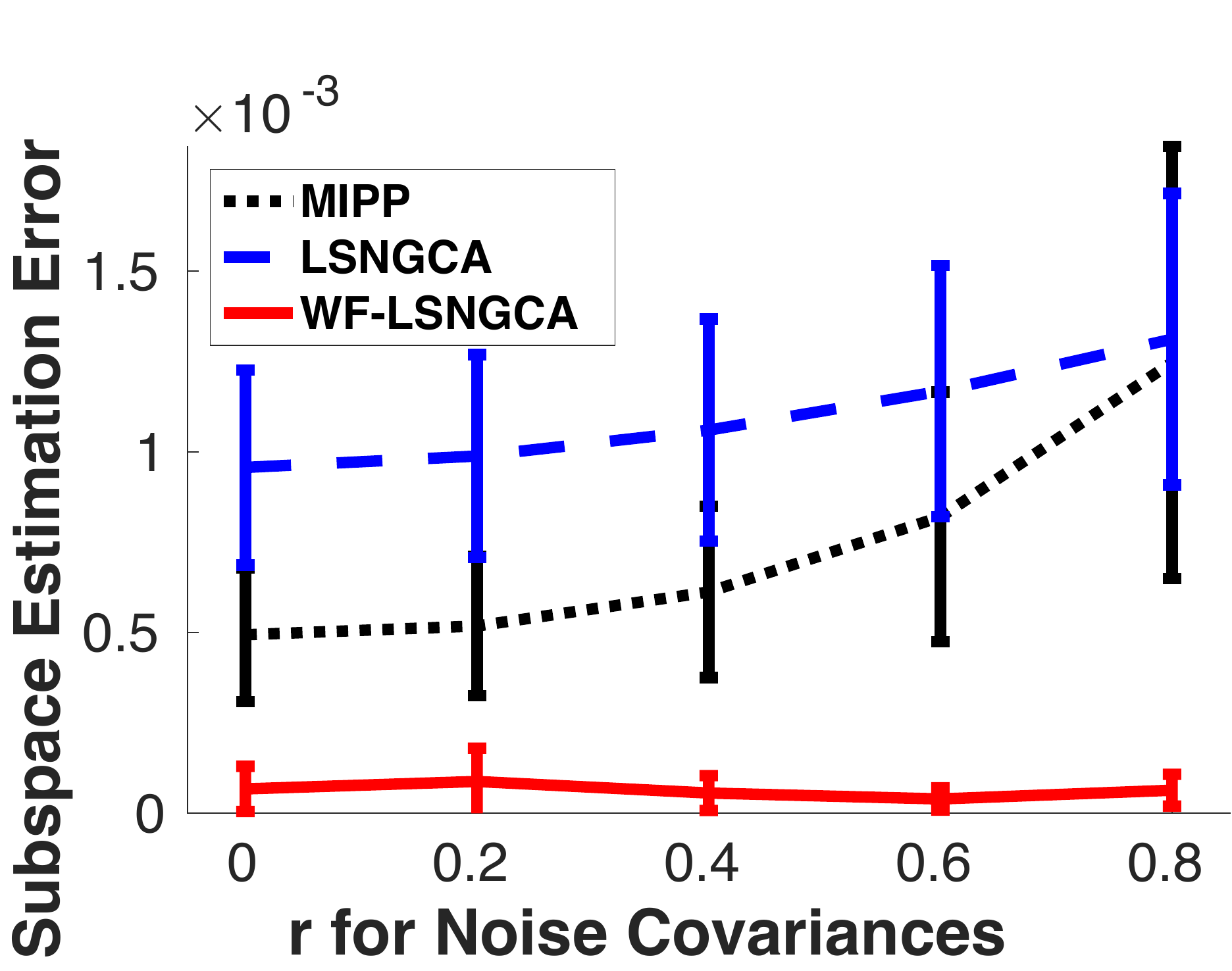}
}
\subfigure[Dependent Super-Gaussian]{
\includegraphics[width = 0.225\textwidth]{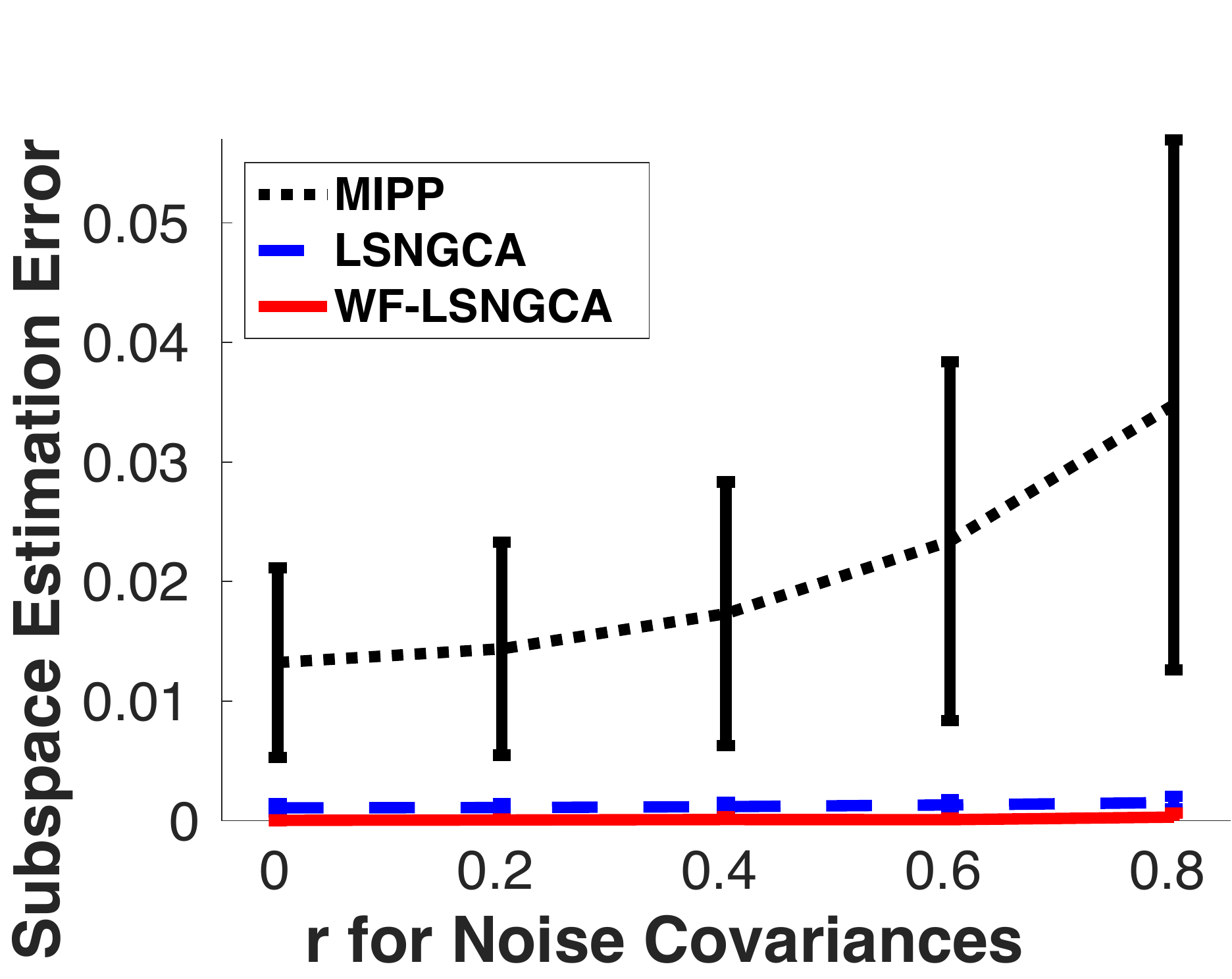}
}
\subfigure[Dependent Sub-Gaussian]{
\includegraphics[width = 0.225\textwidth]{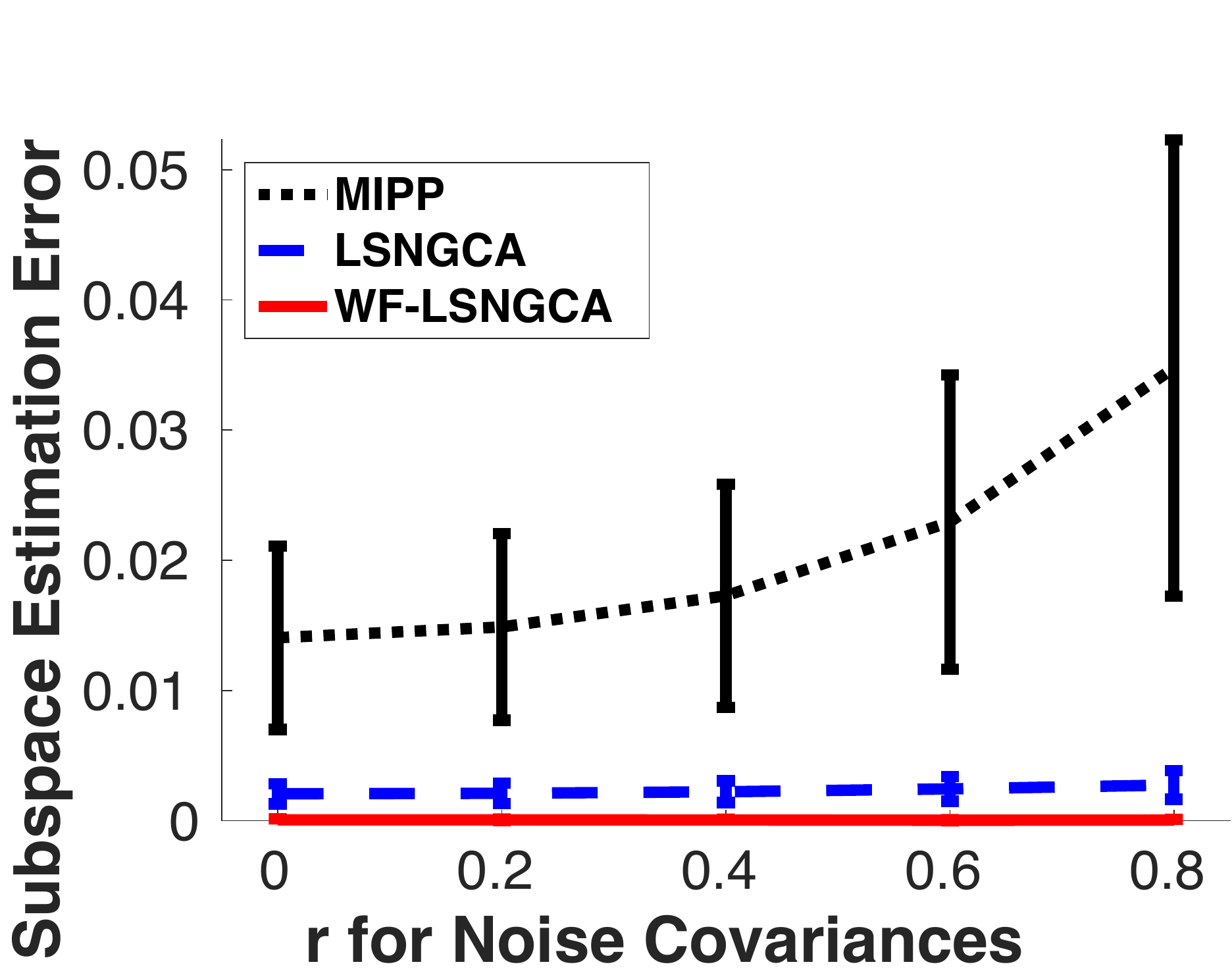}
}
\subfigure[Dependent Super- and Sub-Gaussian]{
\includegraphics[width = 0.225\textwidth]{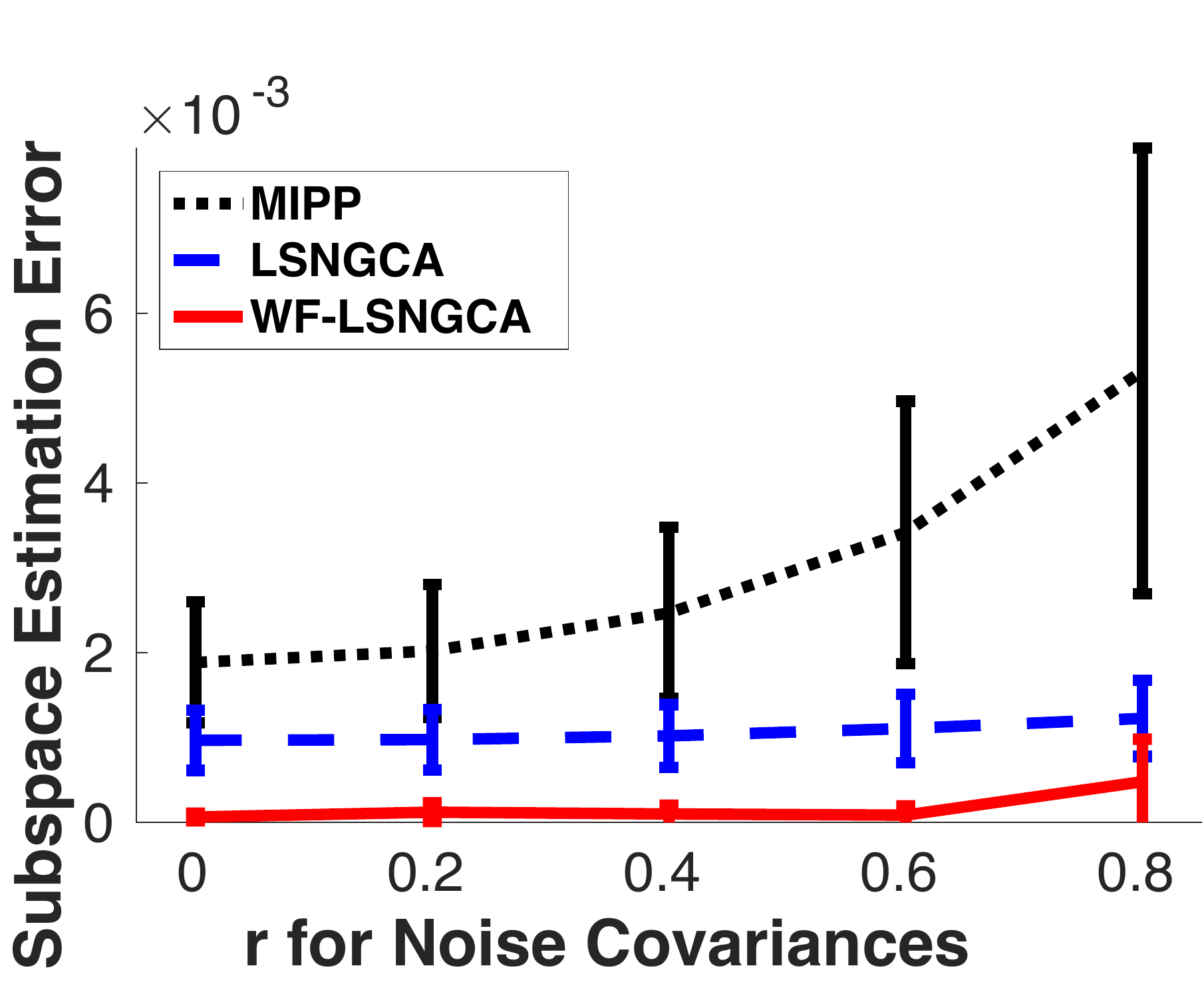}
}
\caption{Averages and standard deviations of the subspace estimation error
 as the function of the condition-number controller $r$
  over 50 simulations on artificial datasets.} 
\label{fig:existings}
\vspace{-2mm}
\end{figure}

Let $\bx=\trans{(s_1,s_2,n_3,\ldots,n_{10})}$,
where $\bs := \trans{(s_1,s_2)}$ are the $2$-dimensional non-Gaussian signal components 
and
$\bn := \trans{(n_3,\ldots,n_{10})}$ are the $8$-dimensional Gaussian noise components.
For the non-Gaussian signal components,
we consider the following four distributions 
plotted in Figure~\ref{fig:Dist}:
\begin{description}
\item 
\textbf{(a) Independent Gaussian Mixture}:\\
$\displaystyle p(s_1,s_2) \propto \prod_{i=1}^{2}\left( \exp{-\frac{(s_i-3)^2}{2}}+\exp{-\frac{(s_i+3)^2}{2}}\right)$.
\item 
\textbf{(b) Dependent super-Gaussian}:\\
$p(\bs) \propto \exp{-\norm{\bs}}$.
\item 
\textbf{(c) Dependent sub-Gaussian}:\\
$p(\bs)$ is the uniform distribution on $\left\{\bs\in\mathbb{R}^2|\norm{\bs}\leq1\right\}$.
\item 
\textbf{(d) Dependent super- and sub-Gaussian}:\\
 $p(s_1) \propto \exp{-\abs{s_1}}$ and $p(s_2)$ is the uniform distribution 
on $[c,c+1]$, where $c=0$ if $\abs{s_1}\leq\log2$ and $c=-1$ otherwise.
\end{description}
%
%
For the Gaussian noise components, we include a certain parameter $r\ge0$,
which controls the condition number;
the larger $r$ is, the more ill-posed the data covariance matrix is.
The detail is described in Appendix~\ref{appendix:artificial}. 

We generate $n=2000$ samples for each case,
and standardize each element of the data before applying NGCA algorithms.
The performance of NGCA algorithms is measured by the following \emph{subspace estimation error}:
\begin{align}
\varepsilon(E,\widehat{E}) := \frac{1}{2}\sum^{2}_{i=1}\norm{\widehat{\be}_i-\Pi_{E}\widehat{\be}_i}^2,
\end{align}
where $E$ is the true non-Gaussian index space, $\widehat{E}$ is its estimate,
$\Pi_{E}$ is the orthogonal projection on $E$, and $\{\widehat{\be}_i\}_{i=1}^{2}$ is an orthonormal basis in $\widehat{E}$.

The averages and the standard derivations of the subspace estimation error
over $50$ runs for MIPP, LSNGCA, and WF-LSNGCA are depicted in Figure~\ref{fig:existings}.
This shows that, for all 4 cases, the error of MIPP grows rapidly as $r$ increases. 
On the other hand, LSNGCA and WF-LSNGCA perform much stably against the change in $r$.
However, LSNGCA performs poorly for (a).
Overall, WF-LSNGCA is shown to be much more robust against ill-conditioning
than MIPP and LSNGCA.
\begin{figure}[t]
\centering
\subfigure[The function of sample size.]{
\includegraphics[width = 0.4\textwidth]{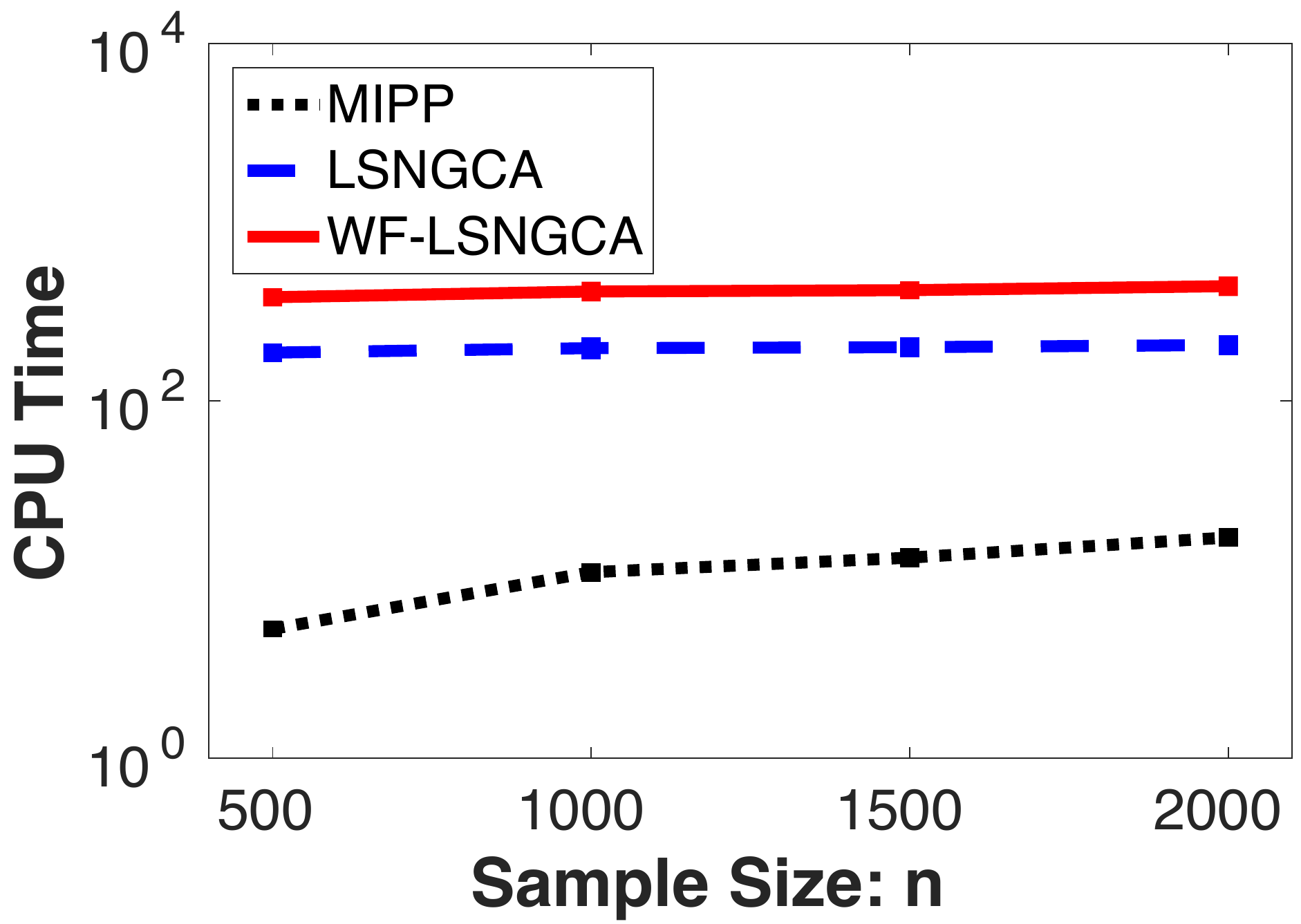}
}
\subfigure[The function of data dimension.]{
\includegraphics[width = 0.4\textwidth]{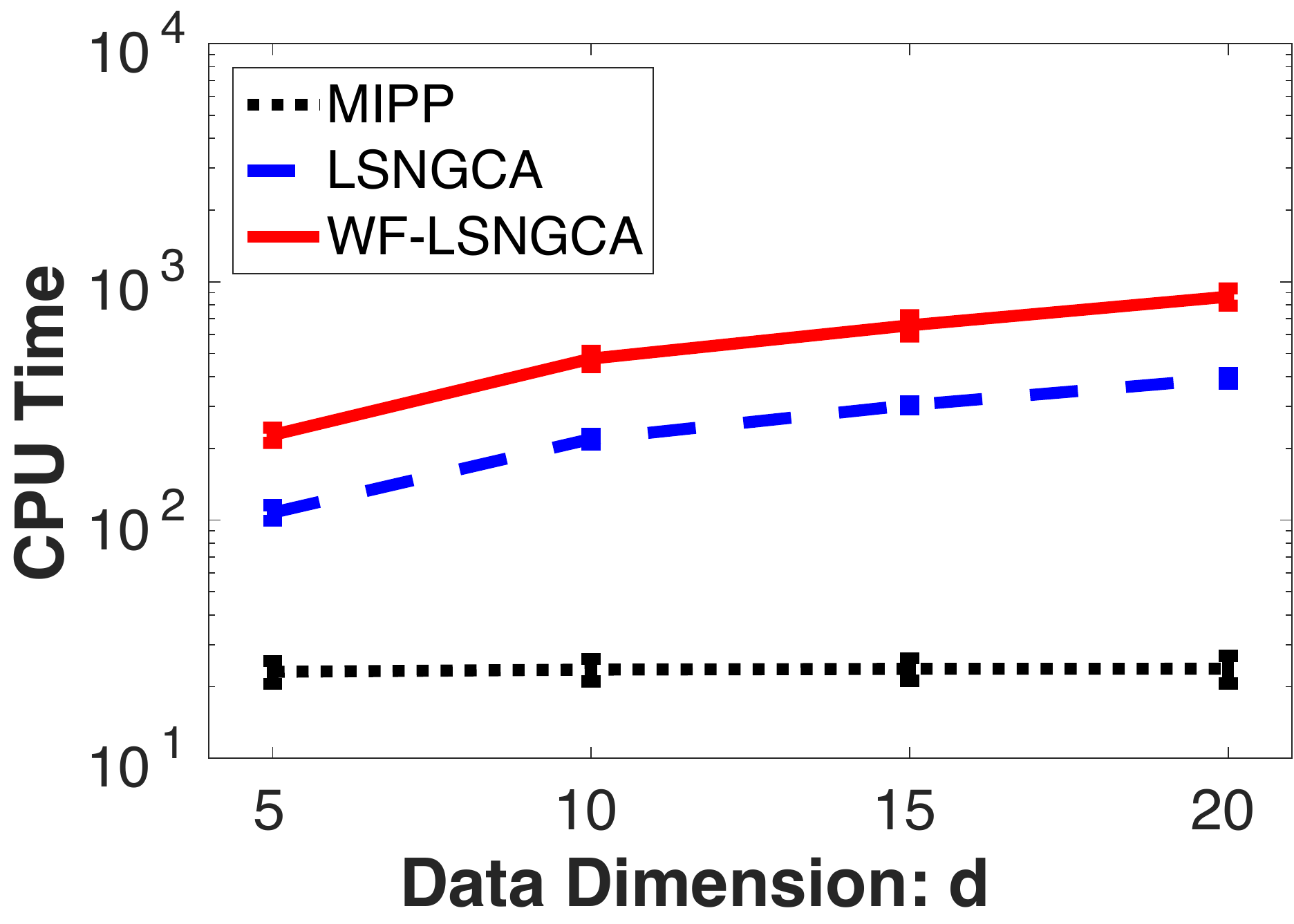}
}
\caption{The average CPU time over 50 runs when the Gaussian mixture is used as non-Gaussian components and the condition-number controller $r=0$. 
The vertical axis is in logarithmic scale.} 
\label{fig:cputime}
\vspace{-2mm}
\end{figure}

In terms of the computation time,
WF-LSNGCA is less efficient than LSNGCA and MIPP,
but its computation time is still just a few times slower than LSNGCA, as seen in Figure~\ref{fig:cputime}.
For this reason, the computational efficiency of WF-LSNGCA would still be acceptable in practice.

\subsection{Benchmark Datasets}
\label{sec:benchmark}

Finally, we evaluate the performance of NGCA methods using
the \emph{LIBSVM binary classification benchmark datasets}\footnote{
We preprocessed the LIBSVM binary classification benchmark datasets as follows:
\begin{itemize}
\item 
\emph{vehicle}: We convert original labels `1' and `2' to the positive label and original labels `3' and `4' to the negative label.
\item 
\emph{SUSY}: We convert original label `0' to the negative label.
\item
\emph{shuttle}: We use only the data labeled as `1' and `4' and regard them as positive and negative labels.
\item 
\emph{svmguide1}: We mix the original training and test datasets.
\end{itemize}
} \cite{chang2011libsvm}.
From each dataset, $n$ points are selected as training (test) samples so that the number of positive and negative samples are equal, and datasets are 
standardized in each dimension.
For an $m$-dimensional dataset, we append $(d-m)$-dimensional noise dimensions following the standard Gaussian distribution
so that all datasets have $d$ dimensions.
Then we use PCA, MIPP, LSNGCA, and WF-LSNGCA 
to obtain $m$-dimensional expressions,
and apply the support vector machine (SVM) \footnote{~We used LIBSVM with MATLAB~\cite{chang2011libsvm}.}
to evaluate the test misclassification rate.
As a baseline, we also evaluate the misclassification rate
by the raw SVM without dimension reduction.

The averages and standard deviations of the misclassification rate over $50$ runs
for $d=50,100$ are summarized in Table~\ref{table:svm}.
As can be seen in the table, 
the appended Gaussian noise dimensions have negative effects on each classification accuracy, and thus the baseline has relatively high misclassification rates.
PCA has overall higher misclassification rates than the baseline
since a lot of valuable information for each classification problem is lost.
Among the NGCA algorithms,  WF-LSNGCA overall compares favorably with the other methods.
This means that it can find valuable low-dimensional expressions for each classification problem without harmful effects of a pre-whitening procedure.
\begin{table}[h]
  \centering
\small
  \caption{
Averages (and standard deviations in the parentheses) of the misclassification rates
for the LIBSVM datasets over 50 runs.
The best and comparable algorithms judged by the two-sample t-test at the significance level 5\% are expressed as boldface.
}
\label{table:svm}
\begin{tabular}{|@{\ }c@{\ }||@{\ }c@{\ }|@{\ }c@{\ }|@{\ }c@{\ }|@{\ }c@{\ }|@{\ }c@{\ }|}
  \hline
&\multicolumn{5}{c@{}|}{$d=50$}\\
  \hline
   \begin{tabular}{@{}c@{}}
Dataset\\
$[m,n]$
\end{tabular}
&
\begin{tabular}{@{}c@{}}
No Dim.\\
Red.
\end{tabular}
& PCA  &  MIPP &
\begin{tabular}{@{}c@{}}
LS\\
NGCA
\end{tabular}
&
\begin{tabular}{@{}c@{}}
WF-LS\\
NGCA
\end{tabular}
\\
    \hline    \hline
vehicle& 0.340& 0.404& 0.328& 0.324& {\bf 0.286}\\
$[18,200]$& (0.038) & (0.034) & (0.044) & (0.044) & {\bf (0.038)}\\
\hline
svmguide3& 0.342  & 0.348  & 0.341 &  0.326 & {\bf 0.308}\\
$[21,200]$& (0.035)  & (0.037)  & (0.041) & (0.039) & {\bf (0.036)}\\
\hline
svmguide1& 0.088  & 0.159  & 0.060 &  0.058 & {\bf 0.053}\\
$[3,2000]$& (0.007)  & (0.012)  & (0.005) & (0.006) & {\bf (0.008)}\\
\hline
shuttle& 0.031  & 0.024  & 0.021 & 0.041 & {\bf 0.007}\\
$[9,2000]$& (0.004)  & (0.007)  & (0.004) & (0.015) & {\bf (0.002)}\\
\hline
SUSY& 0.238 & 0.271 & 0.229 & {\bf 0.223} & 0.228\\
$[18,2000]$& (0.010)  & (0.012)  & (0.010) & {\bf (0.012)} & (0.010)\\
\hline
ijcnn1& 0.102  & 0.273  & 0.084 & 0.061 & {\bf 0.057}\\
$[22,2000]$& (0.007)  & (0.012)  & (0.028) & (0.006) & {\bf (0.007)}\\
    \hline
 \end{tabular}
%
%

\vspace*{5mm}

\begin{tabular}{|@{\ }c@{\ }||@{\ }c@{\ }|@{\ }c@{\ }|@{\ }c@{\ }|@{\ }c@{\ }|@{\ }c@{\ }|}
   \hline
&\multicolumn{5}{@{}c|}{$d=100$}\\
   \hline
   \begin{tabular}{@{}c@{}}
Dataset\\
$[m,n]$
\end{tabular}
&
\begin{tabular}{@{}c@{}}
No Dim.\\
Red.
\end{tabular}
& PCA  &  MIPP &
\begin{tabular}{@{}c@{}}
LS\\
NGCA
\end{tabular}
&
\begin{tabular}{@{}c@{}}
WF-LS\\
NGCA
\end{tabular}
\\
    \hline    \hline
vehicle
& 0.380& 0.432& 0.445& 0.439& {\bf 0.360}\\
$[18,200]$
& (0.033) & (0.033) & (0.036) & (0.045) & {\bf (0.051)}\\
\hline
svmguide3
& 0.363  & 0.367  & 0.443 &  0.427 & {\bf 0.343} \\
$[21,200]$
& (0.034)  & (0.032)  & (0.042) & (0.035) & {\bf (0.044)} \\
\hline
svmguide1
& 0.102  & 0.175  & 0.087 &  0.088 & {\bf 0.067} \\
$[3,2000]$
& (0.008)  & (0.013)  & (0.013) & (0.059) & {\bf (0.020)} \\
\hline
shuttle
& 0.038  & 0.069  & 0.065 & 0.209 & {\bf 0.017}\\
$[9,2000]$
& (0.005)  & (0.013)  & (0.019) & (0.080) & {\bf (0.005)}\\
\hline
SUSY
& 0.250 & 0.280 & 0.234 & {\bf 0.228} & {\bf 0.226}\\
$[18,2000]$
& (0.010)  & (0.011)  & (0.009) & {\bf (0.011)} & {\bf (0.010)}\\
\hline
ijcnn1
& 0.145  & 0.318  & 0.107 & 0.101 & {\bf 0.091} \\
$[22,2000]$
& (0.008)  & (0.013)  & (0.021) & (0.010) & {\bf (0.020)} \\
    \hline
 \end{tabular}
\end{table}

\section{Conclusions}
In this paper, we proposed a novel NGCA algorithm which is computationally efficient,
no manual design of non-Gaussian index functions is required,
and pre-whitening is not involved.
Through experiments, we demonstrated that the effectiveness of the proposed method.


 



\bibliography{references.bib}
\bibliographystyle{unsrt}

\newpage

\appendix
\section*{Supplementary Materials to
Whitening-Free Least-Squares Non-Gaussian Component Analysis}

\section{Details of Artificial Datasets}
\label{appendix:artificial}
Here, we describe the detail of the artificial datasets used in Section~\ref{sec:artificial}.
The noise components $\bn$ are generated as follows:
\begin{enumerate}
\item
The $\bn$ is sampled from the centered Gaussian distribution
with covariance matrix $\mathrm{diag}(10^{-2r},10^{-2r+4r/7},10^{-2r+8r/7},\ldots,10^{2r})$,
where $\mathrm{diag}(\cdot)$ denotes the diagonal matrix. 
\item
The sampled $\bn$ is rotated as $\bn''\in\mathbb{R}^8$by
applying the following rotation matrix $\boldsymbol{R}^{(i,j)}$ for all 
$i,j = 3,\ldots,10$ such that $i<j$:
\begin{align*}
  R^{(i,j)}_{i,i} &= \cos(\pi/4),~~ R^{(i,j)}_{i,j} = -\sin(\pi/4), \\
  R^{(i,j)}_{j,i} &= \sin(\pi/4),~~ R^{(i,j)}_{j,j} = \cos(\pi/4),\\
  R^{(i,j)}_{k,k} &= 1\ (k\neq i,k\neq j),~~ R^{(i,j)}_{k,l} = 0\ (\mathrm{otherwise}).
\end{align*}
\item
The rotated $\bn$ is normalized for each dimension.
\end{enumerate}
By this construction, increasing $r$ corresponds to increasing the \emph{condition number} of the data covariance matrix
(see Figure~\ref{fig:condition-number}).
Thus, the larger $r$ is, the more ill-posed the data covariance matrix is.

\begin{figure}[h]
\centering
\includegraphics[width = 0.4\textwidth]{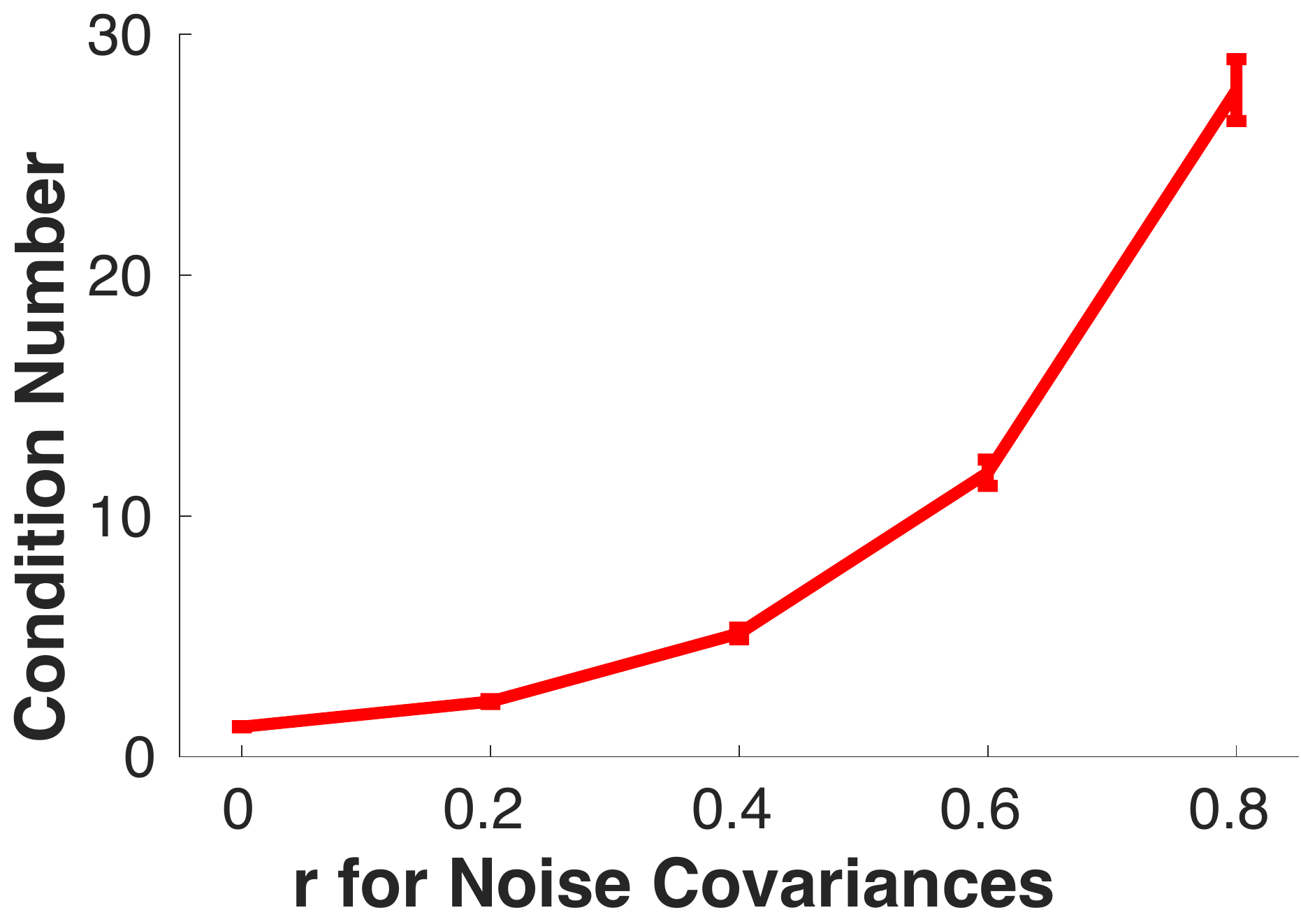}
\caption{Condition number of the data covariance matrix as a function of experiment parameter $r$ (with non-Gaussian components generated from the Gaussian mixture).
}
\label{fig:condition-number}
\end{figure}

\end{document}